\documentclass[pdflatex,sn-basic]{sn-jnl}

\usepackage{graphicx}
\usepackage{multirow}
\usepackage{amsmath,amssymb,amsfonts}
\usepackage{amsthm}
\usepackage{mathrsfs}
\usepackage[title]{appendix}
\usepackage{xcolor}
\usepackage{textcomp}
\usepackage{manyfoot}
\usepackage{booktabs}
\usepackage{bm}
\usepackage{url}

\theoremstyle{thmstyleone}
\newtheorem{proposition}{Proposition}

\newtheorem{observation}{Observation}
\theoremstyle{thmstylethree}

\newcommand{\R}{\mathbb{R}}
\newcommand{\E}{\mathbb{E}}
\newcommand{\Var}{\operatorname{Var}}
\newcommand{\Sig}{\Sigma}
\newcommand{\norm}[1]{\left\lVert #1 \right\rVert}
\newcommand{\LLR}{\ell}

\begin{document}

\title[Gauge-Fixing the Forward--Forward Objective]{Gauge-Fixing the Forward--Forward Objective:
A Whitened Goodness Derived from a Likelihood-Ratio Account}

\author[1]{\fnm{Paolo} \sur{Giannitrapani}}\email{paolo.giannitrapani@uniroma1.it}
\author*[1]{\fnm{Fabrizio} \sur{Frezza}}\email{fabrizio.frezza@uniroma1.it}

\affil[1]{\orgdiv{Department of Information Engineering, Electronics and Telecommunications
(DIET)}, \orgname{Sapienza University of Rome}, \orgaddress{\street{Via Eudossiana 18},
\city{Rome}, \postcode{00184}, \country{Italy}}}

\abstract{The Forward--Forward algorithm trains each layer locally, so that a scalar \emph{goodness} --- the
sum of squared activations --- is high on real inputs and low on contrastive ones. Under an explicit generative model this goodness is the sufficient statistic of a
likelihood-ratio test, and the pairwise form of the objective admits a gauge: a layer can lower its
loss by inflating the scale of its weights rather than by separating the two populations. The
analysis prescribes the repair --- a whitened, scale-invariant goodness trained online within each
layer --- which we evaluate as a training procedure. Across three corpora, three depths and a fourfold range of layer width ($13$ seeds
per cell), it raises linear-probe accuracy over the standard pairwise objective in every measured
cell --- by $4$ to $7$ points on eight of nine corpus-depth combinations --- and closes $16$--$61\%$
of the gap to end-to-end backpropagation. A control isolates the mechanism: Hinton's fixed-threshold loss also bounds the runaway,
to a factor of $1.4$ against $133$, yet tracks the unmodified baseline --- invariance to the gauge,
not a bound on it, is what pays. Against the strongest published alternative --- a sparse, top-$k$ goodness --- the derived
objective is statistically indistinguishable on two corpora of three, yet only it removes the
runaway: sparsity and gauge-invariance are independent axes, and the published variant recovers
accuracy while leaving the pathology in place. We state the boundaries we measured, and every prediction was recorded before its
experiment with the refutations reported.}


\maketitle

\section{Introduction}
\label{sec:intro}

The backpropagation of errors \citep{rumelhart1986learning} that underlies modern deep learning
is difficult to reconcile with known neural mechanisms --- it requires symmetric feedback weights
and a global, two-phase error signal \citep{lillicrap2020backpropagation} --- which has motivated a
sustained search for local learning rules, from feedback alignment
\citep{lillicrap2016random,nokland2016direct} to greedy and gradient-isolated layer-wise training
\citep{nokland2019training,belilovsky2019greedy,lowe2019putting}. Forward--Forward (FF) learning
\citep{hinton2022forward} is a recent and radical member of this family: it replaces the backward
pass with a local, per-layer objective, each layer computes a scalar \emph{goodness} of its activation vector
$\bm a\in\R^d$,
\begin{equation}
G(\bm a)=\norm{\bm a}^2=\sum_{i=1}^d a_i^2,
\label{eq:goodness}
\end{equation}
and is trained so that $G$ exceeds a threshold $\theta$ on \emph{positive} (real) inputs and falls
below it on \emph{negative} (corrupted or contrastive) inputs; the activation is then normalized
and passed to the next layer, which repeats the game. The appeal is architectural --- no weight
transport, no stored activations, constant memory in depth --- and the algorithm has attracted a
family of variants that modify the objective empirically (similarity-based objectives,
predictive-coding hybrids \citep{ororbia2023predictive}, cascaded block-wise predictors
\citep{zhao2023cascaded}, and objectives that couple the otherwise independent layers
\citep{lorberbom2023layer}). Yet the two central design
choices --- \emph{why the sum of squares}, and \emph{why (and how) normalize between layers} ---
have remained without a derivation: the logistic reading of the goodness is assumed rather than
obtained from a model of the activations, the closest theoretical precedent reads the goodness as
a one-class input \emph{density} rather than a two-class \emph{ratio} (\S\ref{sec:related}), and
the growing goodness literature searches the space of functional forms empirically.

This paper gives them a first-principles account, and then tests it. Our starting point is a
question of statistical decision theory: \emph{if a layer must decide ``positive or negative''
from its own activations, what is the optimal scalar to compute?} Modeling the activations
generatively turns the goodness from a heuristic into a testable object: every modeling assumption
(isotropy, Gaussianity, class difference in scale) becomes a measurable quantity, and every
generalization (anisotropy, heavy tails) yields a concrete alternative statistic whose predicted
behavior can be checked on trained networks.

\paragraph{Why the functional form matters for learning systems}
The stakes are not only interpretive. FF is one of the few training rules that dispenses with the
backward pass entirely, and its appeal --- no weight transport, no stored activations, memory
constant in depth --- is precisely what makes it a candidate for settings where backpropagation is
impractical: on-device and edge training under hard memory budgets, neuromorphic and analog
substrates on which a global two-phase error signal cannot be realized, and models of learning
constrained by what cortical circuits can plausibly compute. In every such setting the layer's
scalar objective \emph{is} the learning signal, so the question of which scalar to compute is a
design question, not a stylistic one. It is also a live question: as FF is scaled to deeper
convolutional backbones \citep{jiang2026covariance} and its objective is searched over empirically
\citep{shah2025search}, practitioners must choose a readout, an inter-layer normalization, and a
pairwise or per-sample loss --- and today those choices are made by trial. A derivation converts
each of them into a determinate answer with measurable conditions of validity: which readout is
optimal under which activation statistics (\S\ref{sec:llr}--\S\ref{sec:general}), which
normalization keeps deep goodness from collapsing as width grows (\S\ref{sec:norm}), and which
objective admits a shortcut that lets a network reduce its loss without learning
(\S\ref{sec:norm}). The account below is therefore a statement about how this class of networks
should be built and trained, tested on networks we build and train.

\paragraph{Contributions}
The account yields a training procedure, and the procedure is the paper's main empirical claim:
replacing the plain squared goodness by a whitened, gauge-invariant one raises the quality of the
representation Forward--Forward learns in every configuration we measured, on three corpora, three
depths and a fourfold range of width, and closes a substantial part of the distance to end-to-end
backpropagation (\S\ref{sec:method}). A control separates that gain from the obvious alternative
explanation: Hinton's fixed-threshold loss also bounds the runaway scale and does not gain, so what
pays is invariance to the gauge rather than a bound on it. A second comparison places the result in
the literature: the strongest published alternative, a sparse top-$k$ goodness, matches the derived
objective on two corpora of three while leaving the scale runaway untouched, so the two mechanisms
are independent and only one of them addresses the pathology identified here. The analytical results that motivate the
procedure are:
\begin{enumerate}
\item \textbf{The goodness as a sufficient statistic (\S\ref{sec:llr}).} Under a zero-mean
isotropic Gaussian model in which positives differ from negatives by scale, the log-likelihood
ratio is affine in the squared pre-activation norm $\norm{\bm z}^2$ (notation in
\S\ref{sec:llr}): Hinton's goodness is the sufficient statistic of the optimal
test and the FF threshold is its decision boundary, with a closed form for the optimal $\theta$.
The ReLU makes the deployed post-activation goodness a monotone surrogate of this statistic.
This turns the logistic reading attached to the goodness from an assumption into a derived
statement with explicit conditions (Proposition~\ref{prop:suff}), and
separates the two-class \emph{ratio} account from one-class density readings (\S\ref{sec:related}).
\item \textbf{Structured and heavy-tailed generalizations (\S\ref{sec:general}).} Anisotropic
populations give the Bayes statistic $\bm z^{\!\top}\bm M\bm z$ with
$\bm M=\Sig_-^{-1}-\Sig_+^{-1}$ (the plain square is $\bm M\propto\bm I$). Gaussian Scale
Mixtures (GSM) give a \emph{saturating} statistic $h(q)$ of the Mahalanobis form $q$, whose slope is a
difference of posterior precisions --- divisive normalization
\citep{heeger1992normalization,carandini2012normalization} --- with (i) a finite per-sample
evidence bound $\nu\log(s_+/s_-)$ (Student-$t$ case, closed form), (ii) \emph{no} effect on
single-sample ranking when monotone, and (iii) a provable advantage under aggregation across
locations, which we quantify. The mixer variable is precisely the GSM prior on which the
information-theoretic quality estimator VIF conditions \citep{sheikh2006image}, connecting the FF
goodness to perceptual-quality machinery.
\item \textbf{A theory of the normalization (\S\ref{sec:norm}).} Because the length of $\bm a$
\emph{is} the goodness, passing it on would let the next layer win trivially; normalization forces
each layer to extract conditionally new directional information. We characterize the correct
normalization --- remove the length, preserve unit per-coordinate energy --- prove that unit-norm
($\ell_2$) normalization violates it and collapses deep goodness as width grows, and show that the
pairwise FF objective admits a \emph{scale-inflation shortcut} (loss reducible by weight inflation
alone) that a whitened, negative-metric goodness gauges out.
\item \textbf{A measure-first empirical study (\S\ref{sec:exp}).} On convolutional FF networks
trained to separate clean natural images from distorted ones, with predictions recorded before
each experiment: the measured activation regime is exactly the one the theory flags as adverse for
the plain square (scale ratios $1.0$--$1.2$, participation ratios $0.02$--$0.09$, excess kurtosis
$2.4$--$5.0$); the structured readout delivers the predicted gains and sign correction; the GSM
tail index fitted independently matches the one implied by the measured kurtosis
($\hat\nu\in[4.9,6.8]$ vs.\ $5$--$6.5$); the predicted-null cases come out null \emph{for the
predicted reasons}; and whitened-goodness training both stabilizes the optimization and improves the
learned representation in every configuration measured in \S\ref{sec:method}.
\end{enumerate}

The paper is explanatory and constructive in that order: the identification comes first, and the
training procedure follows from it rather than from a search over objectives. We do not claim
state-of-the-art accuracy for Forward--Forward, which remains well behind end-to-end
backpropagation on every task we measure; the claim is that a specific, derived change to the local
objective recovers a substantial and reproducible part of that gap, and that a control identifies
which property of the change is responsible.

\paragraph{Relation to prior work (in brief)} The mathematical ingredients are classical; the
contribution is their \emph{identification} with the FF objective, and the consequences that
identification carries for readout, normalization, and training (Section~\ref{sec:related}). Two probabilistic readings of the goodness precede us
--- the logistic gloss of \citet{hinton2022forward} and the one-class input-density account of
\citet{zhou2022activation} --- and the recent goodness literature is benchmark-driven; we depart
from all three by modelling the \emph{two} populations FF contrasts and deriving which statistic
is optimal, in which metric, with which failure modes. Section~\ref{sec:related} makes the
positioning precise.

\section{Background and related work}
\label{sec:related}

\paragraph{Forward--Forward learning} FF sits in a lineage of local, backprop-free training. Feedback alignment replaces the
transposed weights of the backward pass with fixed random matrices
\citep{lillicrap2016random,nokland2016direct}; greedy layer-wise and gradient-isolated methods
train each block on a local objective, whether a supervised local loss
\citep{nokland2019training,belilovsky2019greedy} or a self-supervised mutual-information bound
\citep{lowe2019putting,oord2018representation}, and can scale to ImageNet
\citep{belilovsky2019greedy}. FF \citep{hinton2022forward} trains layers greedily with
the local goodness \eqref{eq:goodness} against positive/negative data, normalizing activations
between layers; negatives may be corrupted inputs, hybrid images, or contrastive samples. Reported
performance trails backpropagation on complex tasks, and subsequent work has explored alternative
local objectives and architectures: predictive-coding hybrids with lateral competition
\citep{ororbia2023predictive}, the cascaded forward algorithm that attaches a local predictor to
each block and dispenses with negative samples \citep{zhao2023cascaded}, and objectives that let
the otherwise independently-trained layers collaborate \citep{lorberbom2023layer}. These works modify the objective \emph{empirically}; our contribution is a
\emph{derivation} from which the plain square, structured quadratics, and saturating statistics
appear as special cases of one likelihood-ratio family, together with the conditions under which
each is optimal.

\paragraph{Probabilistic readings of the goodness}
A probabilistic gloss is present from the start: \citet{hinton2022forward} passes $G-\theta$
through a logistic, reading it as a log-odds --- but the reading is assumed, with no stated
activation model under which it is optimal. The closest theoretical precedent is
\citet{zhou2022activation}, which derives, from a local competition rule, that a sum-of-squares
activation estimates the \emph{input probability} $p(\bm x)$: a one-class, generative density
reading (goodness as typicality). The subsequent goodness literature is benchmark-driven,
treating the functional form as a hyperparameter to be searched --- \citet{shah2025search}
benchmark 21 candidates, a registry that even contains a one-class Gaussian energy and a whitened
energy as empirical entries, and that introduces normalized variants to stop the network from
``cheating'' by inflating weight scale. Our account differs from both readings in kind: we model
the \emph{two} activation populations that FF actually contrasts and identify the goodness as the
sufficient statistic of their likelihood \emph{ratio}, under assumptions that are measured rather
than posited (\S\ref{sec:s1}). This is what supplies the closed-form threshold, tells \emph{which}
quadratic is optimal and in \emph{which} metric (\S\ref{sec:general}), derives the saturating
statistic and its aggregation advantage, and formalizes the scale-inflation ``cheat'' as a gauge
symmetry with its removal (Observation~\ref{obs:gauge}) --- none of which is available from a
marginal-density or logistic reading, or from an empirical search over forms.

\paragraph{The goodness-function design space}
The choice of goodness is now an active empirical question. \citet{shah2025search} benchmark
twenty-one goodness functions across four corpora, among them a \emph{whitened energy} computed on
decorrelated features, and find several that outperform the sum of squares;
\citet{yuksel2026sparse} study the design space systematically and report large gains from
selective measurement, identifying sparsity as the dominant design choice. Both proceed by search
and comparison. Our starting point is complementary and, we believe, not otherwise available: the
form of the statistic is \emph{derived} from an explicit generative model rather than selected
empirically, which yields not only a candidate objective but the conditions under which it is
optimal, a mechanism for why the default fails, and a control that identifies which property of the
replacement is doing the work.

\paragraph{Second-order goodness at scale}
In concurrent and independent work, \citet{jiang2026covariance} observe
that the per-channel sum of squares retains only the diagonal of the uncentered covariance and
discards cross-channel structure, and augment the goodness with learnable cross-channel
projections and multi-scale spatial aggregation, reporting consistent per-layer gains on
ImageNet-scale benchmarks. That construction is arrived at empirically and motivated by
scalability; the quadratic family here is instead \emph{derived}, the Mahalanobis form appearing
as the Bayes-optimal statistic under a measured model and used diagnostically rather than as an
architecture. The two lines are complementary, and their result corroborates --- at a scale we do
not reach --- the claim that the structural, off-diagonal channel carries the discriminative
signal that the plain square discards.

\paragraph{Gaussian scale mixtures and divisive normalization} Responses of local filters to
natural images are strongly leptokurtic and are classically modeled as Gaussian scale mixtures
\citep{wainwright2000scale,portilla2003image}; divisive normalization is the canonical cortical
computation \citep{heeger1992normalization,carandini2012normalization}. We show these two
literatures meet inside the FF goodness: the Bayes statistic under a GSM \emph{is} an adaptive,
divisively-normalized energy, and the mixer is the same latent scale on which the VIF quality
estimator conditions \citep{sheikh2006image}.

\paragraph{Hypothesis testing} Our tools are classical --- sufficiency, monotone likelihood
ratios, and the Neyman--Pearson lemma \citep{lehmann2005testing}; the contribution is their
application to the FF objective and the resulting measurable predictions.

\section{The goodness as a likelihood-ratio statistic}
\label{sec:llr}

A layer maps a normalized input $\tilde{\bm x}$ to a pre-activation
$\bm z=\bm W\tilde{\bm x}\in\R^d$ and an activation $\bm a=\phi(\bm z)$ with
$\phi=\mathrm{ReLU}$. FF declares an input positive when $G(\bm a)>\theta$. We model the
pre-activation as the random object and treat the ReLU below.

\paragraph{Isotropic Gaussian model} Assume
\begin{equation}
\bm z\mid + \sim \mathcal N(\bm 0,\sigma_+^2\bm I_d),\qquad
\bm z\mid - \sim \mathcal N(\bm 0,\sigma_-^2\bm I_d),
\label{eq:isomodel}
\end{equation}
i.e.\ the two populations share direction statistics and differ by \emph{scale}. The
Bayes-optimal statistic is the log-likelihood ratio
$\LLR(\bm z)=\log p(\bm z\mid +)/p(\bm z\mid -)$, and a direct computation gives
\begin{equation}
\LLR(\bm z)=\underbrace{\tfrac12\Big(\tfrac1{\sigma_-^2}-\tfrac1{\sigma_+^2}\Big)}_{\alpha}
\,\norm{\bm z}^2+\underbrace{\tfrac d2\log\tfrac{\sigma_-^2}{\sigma_+^2}}_{\beta}.
\label{eq:llriso}
\end{equation}

\begin{proposition}[Goodness as sufficient statistic]
\label{prop:suff}
Under \eqref{eq:isomodel} with $\sigma_+>\sigma_-$, the LLR is strictly increasing in
$\norm{\bm z}^2$; hence $\norm{\bm z}^2$ is a sufficient statistic for the decision and the
optimal rule is ``positive iff $\norm{\bm z}^2>\tau$'' with
$\tau=-\beta/\alpha=d\log(\sigma_+^2/\sigma_-^2)\big/(\sigma_-^{-2}-\sigma_+^{-2})$.
\end{proposition}

Hinton's goodness is therefore not arbitrary: it is the exact sufficient statistic of an
isotropic-Gaussian scale test, and the FF threshold plays the role of the Neyman--Pearson
boundary. The single substantive assumption that makes the \emph{plain} goodness optimal is that
positives excite the learned filters more than negatives ($\sigma_+>\sigma_-$) \emph{isotropically}
--- an assumption we will measure. Hinton's own formulation already passes $G-\theta$ through a
logistic \citep{hinton2022forward}; Proposition~\ref{prop:suff} states the generative conditions
under which that log-odds reading is exactly Bayes-optimal, and supplies the boundary in closed
form. It also separates this two-class account from the one-class reading of
\citet{zhou2022activation}, in which a squared activation estimates the input density: the layer's
optimal scalar is a ratio between the two populations FF contrasts, not a marginal typicality.

\begin{proposition}[The deployed post-ReLU goodness]
\label{prop:relu}
FF evaluates the goodness on $\bm a=\mathrm{ReLU}(\bm z)$, so $G(\bm a)=\sum_i z_i^2\,\mathbf 1[z_i>0]$.
Under the isotropic scale model $\bm z\mid\pm\sim\mathcal N(\bm 0,\sigma_\pm^2\bm I)$:
\emph{(i)} $\mathrm{ReLU}$ is positively homogeneous, so $G$ is exactly scale-equivariant,
$G(\lambda\bm z)=\lambda^2G(\bm z)$ for $\lambda>0$, and the law of $G$ indexed by $\sigma^2$ is a
scale family; \emph{(ii)} the direction $\bm a/\norm{\bm a}$ is ancillary for $\sigma$, so
$\norm{\bm a}^2$ is sufficient for $\sigma$ given $\bm a$; \emph{(iii)} the family is
stochastically increasing in $\sigma$, so the FF threshold test on the deployed goodness is
unbiased with monotone power. Its likelihood ratio is increasing except on a lower-tail set ---
near-total rectification --- whose mass decays rapidly with width ($3\times10^{-2}$ at $d=16$,
$3\times10^{-5}$ at $d=64$, $2\times10^{-16}$ at $d=256$, for $\sigma_+/\sigma_-=1.3$): the
threshold test is uniformly most powerful outside a set of vanishing probability, and exactly so in
the wide-layer limit.
\end{proposition}

Rectification is not free. With $X=z^2\mathbf 1[z>0]$ one has $\E X=\tfrac12$ and $\E X^2=\tfrac32$,
so $\E G_0=d/2$ and $\Var G_0=5d/4$: the coefficient of variation rises from $\sqrt{2/d}$ for
$\norm{\bm z}^2$ to $\sqrt{5/d}$ for the deployed statistic (measured $0.176$ and $0.280$ at
$d=64$), costing AUC $0.87\to0.76$ at a scale ratio of $1.15$. The pre-activation analysis therefore
transfers to the statistic FF actually thresholds, at a quantified loss of power; our measurements
address $\bm z$, and Proposition~\ref{prop:relu} is what licenses that.

\section{Structured and heavy-tailed generalizations}
\label{sec:general}

\subsection{Anisotropy: the Mahalanobis goodness}
Dropping isotropy, $\bm z\mid\pm\sim\mathcal N(\bm 0,\Sig_\pm)$ gives
\begin{equation}
\LLR(\bm z)=\tfrac12\,\bm z^{\!\top}\bm M\,\bm z+\mathrm{const},
\qquad \bm M=\Sig_-^{-1}-\Sig_+^{-1}.
\label{eq:maha}
\end{equation}
The optimal goodness is a structured quadratic; the plain square is the special case
$\bm M\propto\bm I$ (common eigenbasis, pure scale difference). Whether a trained network lives in
that special case is an empirical question, answered by the generalized eigenvalues of
$(\Sig_+,\Sig_-)$: all equal $\Leftrightarrow$ pure scale $\Leftrightarrow$ plain square optimal;
spread $\Leftrightarrow$ structural difference $\Leftrightarrow$ Mahalanobis motivated. Since
$\bm M$ is generally indefinite, an energy-like variant is obtained by whitening,
$G=\norm{\bm L\bm z}^2$ with $\bm L^{\!\top}\bm L=\Sig_-^{-1}$, which will reappear in
\S\ref{sec:norm} as a trainable objective. The two are distinct statistics, and the condition under
which they agree is exactly the anisotropic form of the pure-scale hypothesis.

\begin{proposition}[When the whitened form is the likelihood ratio]
\label{prop:whiten}
The whitened goodness $\norm{\bm L\bm z}^2$ is positive semi-definite, whereas $\bm M$ is indefinite
in general, so the two are different statistics. They coincide up to a positive factor precisely
when the populations differ by a pure scale in the negative metric, $\Sig_+=c\,\Sig_-$ with $c>1$:
then $\bm M=(1-1/c)\,\Sig_-^{-1}$ is positive definite and
$\bm z^{\!\top}\bm M\bm z=(1-1/c)\norm{\bm L\bm z}^2$, so the two induce the same ranking and the
same test. Outside that condition the whitened form is a gauge-fixed proxy for the likelihood ratio,
not the likelihood ratio itself.
\end{proposition}

\noindent This is the anisotropic counterpart of the isotropic case that motivates the plain square:
$\bm M\propto\bm I$ recovers $\norm{\bm z}^2$, $\bm M\propto\Sig_-^{-1}$ recovers the whitened
energy, and the general indefinite $\bm M$ recovers neither.

\paragraph{Prediction residuals as members of the family}
Classical structural statistics obtained as \emph{prediction residuals} belong to the same
family: if $\bm P$ is a linear self-predictor of the field from a local basis --- e.g., the
weighted least-squares Hermite--Gauss structural prediction at the core of the partial-reference
estimator of \citet{giannitrapani2026partial} --- the residual energy
$\norm{(\bm I-\bm P)\bm x}^2$ is a quadratic form whose matrix is induced by a \emph{prior,
one-class model} of the positive (natural) population, rather than measured from the two
populations. Such statistics are thus model-based points in the family, between the plain square
($\bm M\propto\bm I$) and the fully measured Mahalanobis form: the account places hand-designed
structural channels and learned FF goodnesses on one axis --- which quadratic, in which metric,
chosen how.

\subsection{Heavy tails: the GSM likelihood ratio}
\label{sec:gsm}
Let $\bm z=\sqrt u\,\bm g$, $\bm g\sim\mathcal N(\bm 0,\Sig)$, $u\sim p(u)$ --- the Gaussian scale
mixture that describes local natural-image statistics \citep{wainwright2000scale}. With
$q(\bm z)=\bm z^{\!\top}\Sig^{-1}\bm z$,
\begin{equation}
\begin{aligned}
p(\bm z)&=(2\pi)^{-d/2}\lvert\Sig\rvert^{-1/2}F(q),\\[2pt]
F(q)&=\E_u\big[u^{-d/2}e^{-q/(2u)}\big],
\end{aligned}
\end{equation}
so for two classes with common structure and mixers $p_\pm$ the LLR is a scalar function of the
quadratic form, $\LLR=h(q)=\log F_+(q)-\log F_-(q)$.

\begin{proposition}[Adaptive precision; monotonicity]
\label{prop:gsm}
$-2\,\frac{d}{dq}\log F_\pm(q)=\E_\pm[u^{-1}\mid q]$, the posterior mean precision of the mixer
given the observation. Hence
\[
\begin{aligned}
h'(q)&=\tfrac12\big(\E_-[u^{-1}\mid q]-\E_+[u^{-1}\mid q]\big),\\[2pt]
h(q)&=\tfrac12\int_0^q\big(\E_-[u^{-1}\mid t]-\E_+[u^{-1}\mid t]\big)\,dt,
\end{aligned}
\]
and if $p_+/p_-$ is nondecreasing in $u$ (MLR order: positives favor larger scales) then $h$ is
nondecreasing. The optimal statistic accrues evidence as energy weighted by an adaptive posterior
precision --- divisive normalization; the mixer $u$ is the latent scale on which VIF conditions.
\end{proposition}

\begin{proof}[Proof sketch]
Differentiate $F$ under the integral; the tilted density $\propto u^{-d/2}e^{-q/2u}p_\pm(u)$ is
the posterior of $u$. Tilting both mixers by the same positive weight preserves their likelihood
ratio, so the MLR order passes to the posteriors, implying stochastic dominance; $u^{-1}$ is
decreasing, ordering the two expectations. Full details in \ref{app:proofs}.
\end{proof}

\paragraph{Worked example (Student-$t$)} With $u\sim\mathrm{InvGamma}(\nu/2,\nu/2)$ and scales
$s_+>s_-$ at common $\nu,\Sig$ (write $r=s_+/s_-$, $s_-{=}1$):
\begin{equation}
\begin{aligned}
h(q)&=-d\log r+\frac{\nu+d}{2}\Big[\log\big(1+\tfrac q\nu\big)-\log\big(1+\tfrac q{\nu r^2}\big)\Big]\\[2pt]
&\xrightarrow{\ q\to\infty\ }\ \nu\log r .
\end{aligned}
\label{eq:tllr}
\end{equation}
The LLR \emph{saturates} (Fig.~\ref{fig:hq}): one observation carries at most $\nu\log r$ nats of
evidence, however large its energy, because under heavy tails an extreme activation is common
under both hypotheses; as $\nu\to\infty$, \eqref{eq:tllr} recovers the affine Gaussian LLR
\eqref{eq:llriso}.

\begin{figure}[t]
\centering
\includegraphics[width=3.0in]{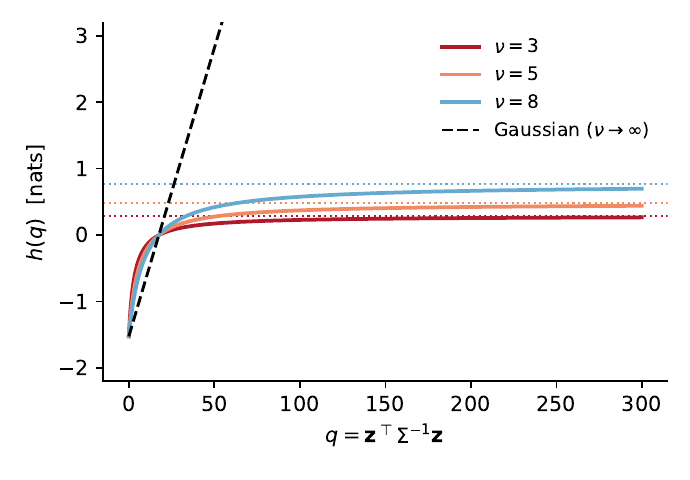}
\caption{The Student-$t$ GSM likelihood ratio $h(q)$ ($d{=}16$, $r{=}1.10$): saturating growth with
finite plateau $\nu\log r$ (dotted), against the unbounded affine Gaussian LLR (dashed).}
\label{fig:hq}
\end{figure}

\begin{figure}[t]
\centering
\includegraphics[width=3.0in]{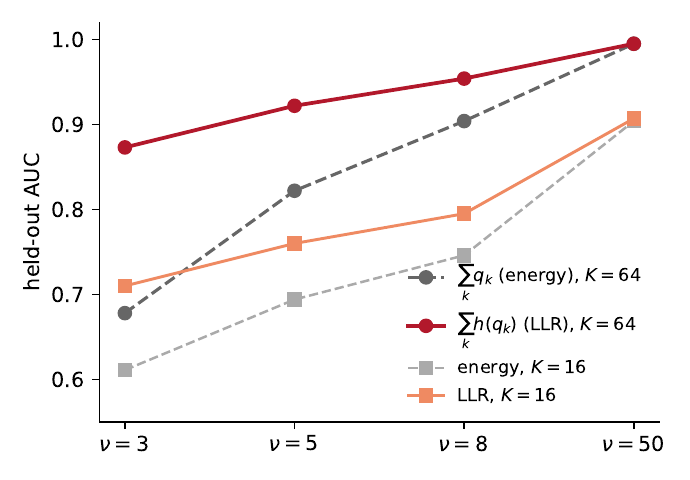}
\caption{Aggregation over $K$ i.i.d.\ locations (synthetic, $N{=}3000$ per class): the
Neyman--Pearson sum $\sum_k h(q_k)$ vs.\ the energy sum $\sum_k q_k$. The gap grows with tail
heaviness and $K$, and vanishes in the Gaussian limit.}
\label{fig:agg}
\end{figure}

\begin{proposition}[Where the nonlinearity matters]
\label{prop:agg}
(i) \emph{Single sample:} if $h$ is monotone, thresholding $h(q)$ and thresholding $q$ induce the
same ROC: heavy tails change evidence calibration, not single-sample ranking. (ii)
\emph{Aggregation:} for $K$ i.i.d.\ locations the Neyman--Pearson statistic is $\sum_k h(q_k)$,
not $\sum_k q_k$; the energy sum is unbounded in any single location and dominated by outliers,
while per-location evidence is bounded in $[-d\log r,\ \nu\log r]$, so the gap grows with tail
heaviness and with $K$.
\end{proposition}

Numerically (Fig.~\ref{fig:agg}; parameters $d{=}16$, $r{=}1.10$): the plateau matches
$\nu\log r$ to four decimals; $\mathrm{AUC}(q)=\mathrm{AUC}(h(q))$ exactly on single samples
(AUC: area under the ROC curve; $0.5$ is chance); and
the aggregation gap is $+0.195$ AUC at $\nu{=}3,K{=}64$, $+0.099$ at $\nu{=}5$, $+0.000$ at
$\nu{=}50$. These synthetic magnitudes calibrate what to expect on real activations once $\nu$ and
the scale ratio are \emph{measured} (\S\ref{sec:exp}).

\paragraph{What a partial reference adds}
Proposition~\ref{prop:gsm} contains, in passing, an account of why magnitude estimation requires
a reference --- and of how little reference suffices. The slope $\E_\pm[u^{-1}\mid q]$ is the
detector's \emph{internal, test-side} estimate of the local content scale: marginally, content
and corruption are confounded in $q$, and the posterior over the mixer is the best the distorted
image alone can supply --- which is precisely why the statistic detects presence without ranking
magnitude (\S\ref{sec:multi}). A partial reference replaces this posterior with an
\emph{observation}: the clean-side value of the same scalar, one number per locality, turning the
marginal test into a conditional one in which departures of $q$ from a content-matched baseline
are attributable to the corruption. This is the design logic of reference-based quality
assessment in the GSM family --- VIF conditions on exactly this latent scale
\citep{sheikh2006image}, and partial-reference estimators reduce the reference to scalars of this
kind \citep{giannitrapani2026partial}. The detection--estimation boundary is therefore not a wall
but a coordinate statement: what is missing for magnitude is one scalar of reference context per
locality, not the reference image. We leave the corresponding experiment --- conditioning the
learned goodness on a single reference scalar --- to future work.

\section{Normalization across depth, and a gauge problem in the objective}
\label{sec:norm}

\paragraph{Why normalize} Because $\norm{\bm a}^2$ is (a surrogate of) the sufficient statistic,
the length of $\bm a$ already encodes the layer's decision. Passed on raw, the next layer could
make its goodness large by reading the incoming length, learning nothing conditionally new.
Passing the direction forces each layer to extract discriminative information beyond the running
lengths --- informally, a chain decomposition
\begin{equation}
\begin{aligned}
&I(\text{label};\bm a_1,\dots,\bm a_L)\ \approx\\
&\quad\sum_{k=1}^L I\big(\text{label};\hat{\bm a}_k\,\big|\,\norm{\bm a_1}^2,\dots,\norm{\bm a_{k-1}}^2\big),
\end{aligned}
\label{eq:chain}
\end{equation}
an informal statement of the design intent --- each layer should contribute information beyond
the running lengths --- which the normalization implements; we do not use \eqref{eq:chain}
quantitatively.

\paragraph{Which normalization} Unit-norm ($\ell_2$) and RMS normalization pass the same
direction and differ only in scale: $\hat{\bm a}=\bm a/\norm{\bm a}$ hands the next layer
per-coordinate energy $1/d$, while $\tilde{\bm a}=\bm a/\mathrm{rms}(\bm a)$ hands it $1$.

\begin{proposition}[Depth stability]
\label{prop:depth}
If the next layer's pre-activation scales linearly with its input norm and its goodness is the
mean squared activation, then under $\ell_2$ normalization the downstream goodness is
$\Theta(1/d)$ and collapses as width grows, while under RMS normalization it is scale-stable.
\end{proposition}

The statement is relative to the protocol FF fixes: $\ell_2$ and RMS normalization differ by the
factor $\sqrt d$, so the collapse is what happens when a width-independent threshold and
initialization scale are held fixed while $d$ grows --- a statement about a scale convention and the
hyperparameters kept constant beside it, not an intrinsic property of the representation. That is
nonetheless the operative regime, since FF fixes the threshold and the initialization independently
of width. It is a falsifiable design statement, and it is the failure mode we in fact encountered:
training with $\ell_2$ normalization produced deep-layer goodness $\approx 0$ at widths $64$,
restored to $O(1)$ by RMS --- the correct normalization removes the length while preserving unit
per-coordinate energy. The same pair of requirements --- normalized inputs and a
magnitude-preserving nonlinearity, with weight norms self-regularizing --- is reached
independently by \citet{zhou2022activation} from a reconstruction argument: two derivations of one
normalization necessity from different starting points. The requirement is also visible in
practice: scaling FF to deep convolutional backbones,
\citet{jiang2026covariance} adopt RMS normalization and RMS pooling precisely because they
preserve the energy relation on which their goodness depends --- the scale-stable choice
identified by Proposition~\ref{prop:depth}, arrived at from engineering necessity.

\paragraph{A scale-inflation shortcut in the pairwise objective}
FF is often trained with the pairwise (contrastive) loss
$\mathcal L=\mathrm{softplus}(g_--g_+)$ on the goodness. ReLU networks are positively homogeneous:
under $\bm W\to\lambda\bm W$ (one layer), $\bm z\to\lambda\bm z$ and $g\to\lambda^2 g$, so any
positive gap is amplified as $\lambda^2(g_+-g_-)$ and the loss \emph{decreases along the ray of
weight inflation without new information}.

\begin{observation}[Gauge fixing by whitening]
\label{obs:gauge}
The negative-metric goodness $G(\bm z)=\tfrac1C\overline{\bm z^{\!\top}\hat\Sig_-^{-1}\bm z}$,
with $\hat\Sig_-$ a running (EMA) covariance of negative pre-activations, is invariant under
$\bm W\to\lambda\bm W$ at metric equilibrium ($\hat\Sig_-\to\lambda^2\hat\Sig_-$): the
scale-inflation shortcut is gauged out. (Hinton's fixed-threshold loss also bounds the shortcut,
through the $g_-$ term; the pairwise variant does not.)
\end{observation}

The benchmark literature registers the symptom --- normalized-energy variants are introduced
precisely to stop the network from ``cheating'' by inflating weight scale \citep{shah2025search};
Observation~\ref{obs:gauge} supplies the mechanism and the gauge that removes it.
Observation~\ref{obs:gauge} predicts a qualitative difference in training dynamics --- unbounded
goodness growth for the plain pairwise objective, a plateau for the whitened one --- which
\S\ref{sec:exp} confirms, and it realizes the whitened quadratic of \S\ref{sec:general} as a
\emph{trainable} objective rather than a post-hoc readout.

\section{Empirical study}
\label{sec:exp}

\paragraph{Protocol} All experiments follow a measure-first protocol: the quantities that decide
between the models of \S\ref{sec:llr}--\S\ref{sec:general} are measured on trained networks, and
the outcome of each subsequent experiment is predicted in writing before it is run (predictions
are reproduced in the text below); null outcomes are reported as findings. Code, commands and the prediction log are archived (\ref{app:repro}).

\paragraph{Setup} Convolutional FF with three layers ($32/64/64$ channels, stride $2$,
$5{\times}5$ then $3{\times}3$ kernels), RMS inter-layer normalization
(Proposition~\ref{prop:depth}), pairwise loss, Adam \citep{kingma2015adam}, $64{\times}64$ RGB
inputs; positives are Imagenette \citep{howard2019imagenette} training images ($300$/class);
negatives are distorted versions of natural images (Gaussian blur, JPEG, JPEG\,2000, white noise)
in two arms: \emph{naive} (random uncalibrated severities, regenerated each epoch) and
\emph{curated} (a fixed pool with severities calibrated to the informative band of a perceptual
scale by the partial-reference estimator of \citet{giannitrapani2026partial}; details in
\ref{app:repro}).
Representation quality is a linear probe on the $10$ Imagenette classes ($1000$ validation
images); perceptual alignment is the Spearman correlation (SROCC) between per-image goodness and
human DMOS on the CSIQ dataset \citep{larson2010most} ($n{=}750$; the contrast subset is excluded
as out of the distortion domain). Statistics use pre-activations (Proposition~\ref{prop:relu}), pooled
over random spatial locations, with class covariances estimated per layer; readout evaluations
split images $50/50$ into fit and held-out halves. A moderate-scale natural-image corpus is
appropriate here because the object under test is the \emph{activation regime} the theory speaks
to, not benchmark accuracy: what the predictions require is that trained networks exhibit the
relevant scale, anisotropy, and tail structure, which Table~\ref{tab:s1} confirms they do; a
larger corpus would raise cost without changing what is being tested. (This matches the scale of
the FF literature it engages, which is developed on MNIST- and CIFAR-scale data.) The evaluation of
the derived objective in \S\ref{sec:method} is broader: three corpora, three depths and three
widths.

\subsection{The measured activation regime}
\label{sec:s1}

Table~\ref{tab:s1} reports, per layer and arm: the goodness ratio
$\E[G\mid+]/\E[G\mid-]$ (the scale premise of Proposition~\ref{prop:suff}); the excess kurtosis of
pre-activations (Gaussian vs.\ GSM); the participation ratio of the covariance spectrum
(isotropy); and the spread (p90/p10) of the generalized eigenvalues of $(\Sig_+,\Sig_-)$ (pure
scale vs.\ structural difference).

\begin{table}[t]
\caption{Activation statistics of trained FF networks. The regime is adverse for the plain
square: thin scale channel, heavy tails, strong anisotropy, structural class difference.}
\label{tab:s1}
\centering\footnotesize
\setlength{\tabcolsep}{1.0pt}
\begin{tabular}{llccccc}
\toprule
arm & layer & $\E[G|+]/\E[G|-]$ & kurt$_+$ & kurt$_-$ & partic.\ ratio & gen-eig spread\\
\midrule
naive   & L0 & 1.04 & $+2.4$ & $+3.5$ & 0.09 & 5.8\\
naive   & L1 & 1.10 & $+2.5$ & $+2.7$ & 0.04 & 2.5\\
naive   & L2 & 1.24 & $+3.4$ & $+5.0$ & 0.06 & 5.2\\
curated & L0 & 1.02 & $+0.7$ & $+1.1$ & 0.08 & 4.1\\
curated & L1 & 1.01 & $+0.3$ & $+0.3$ & 0.02 & 6.6\\
curated & L2 & 1.01 & $+0.8$ & $+0.6$ & 0.03 & 4.9\\
\bottomrule
\end{tabular}
\end{table}

Every diagnostic lands on the adverse side for the plain square: the scale premise is thin
(ratios $1.0$--$1.2$; essentially absent for perceptually mild negatives), activations are
heavy-tailed (GSM territory; the implied Student-$t$ index from
$\kappa=6/(\nu-4)$ is $\nu\approx5$--$6.5$), covariances are strongly anisotropic, and the class
difference is \emph{structural}, not a scale change. By \eqref{eq:maha}, the isotropic goodness
sees only the (nearly empty) trace channel of this difference. This measurement, made first,
generates the predictions tested next.

\subsection{The structured readout}
\label{sec:gate1}

\emph{Recorded prediction: the Mahalanobis readout will clearly exceed the energy readout on
held-out clean-vs-distorted AUC (the measured spread demands it); on the perceptual axis honesty
requires agnosticism, with $|\rho|<0.3$ read as ``structure discriminates presence, not
severity''.}

Post hoc, with $\bm M$ fitted on the fit half (ridge $10^{-3}$): Table~\ref{tab:gates}. The structured readout improves AUC at every layer of the naive arm ($0.616\to0.715$ at L0); in
the curated arm the energy channel is empty and all readouts sit at chance ($0.49$--$0.54$), so
the comparison there is uninformative. It also --- a
theoretically meaningful detail --- \emph{restores the required sign} of the perceptual
correlation at all layers (the LLR direction demands cleaner $\Rightarrow$ higher goodness
$\Rightarrow$ negative $\rho$; the raw energy even carries the wrong sign at depth). Magnitudes
stay modest, matching the recorded prediction above: the structural channel discriminates the presence of distortion, not its magnitude
($|\rho|\le0.16$). On the detection axis the absolute level is itself informative: the optimal
quadratic readout of \emph{learned} FF features lands in the same weak-to-modest band (held-out
AUC $0.51$--$0.72$) that we measure for \emph{learned} readouts on the \emph{fixed} Hermite--Gauss
structural-prediction basis of \citet{giannitrapani2026partial}, in side probes on LIVE R2 and
CSIQ ($0.63$--$0.82$, JPEG hardest; protocol in \ref{app:repro}, per-cell values with the code). Two routes that place
the learning at opposite ends --- fixed prior features with a learned readout, learned features
with the Bayes-optimal quadratic readout --- meet in one band: the pattern an information ceiling
of local structure predicts, and a representation limit does not.

\begin{table*}[t]
\caption{Detection and perception for the three readouts: held-out AUC (clean vs.\ distorted) and CSIQ SROCC;
$\hat\nu_\pm$ are the tail indices fitted per class. Bold: the theoretically decisive entries.}
\label{tab:gates}
\centering\small
\setlength{\tabcolsep}{3.5pt}
\begin{tabular}{llcc ccc ccc}
\toprule
& & \multicolumn{2}{c}{$\hat\nu$} & \multicolumn{3}{c}{AUC} & \multicolumn{3}{c}{SROCC vs DMOS}\\
\cmidrule(lr){3-4}\cmidrule(lr){5-7}\cmidrule(lr){8-10}
arm & layer & $+$ & $-$ & energy & Mahal. & $t$-LLR & energy & Mahal. & $t$-LLR\\
\midrule
naive & L0 & 5.2 & 5.3 & 0.616 & \textbf{0.715} & 0.715 & $-0.100$ & $-0.104$ & $-0.108$\\
naive & L1 & 5.5 & 5.2 & 0.611 & 0.655 & 0.645 & $-0.069$ & $\bm{-0.121}$ & $-0.122$\\
naive & L2 & 6.8 & 6.1 & 0.613 & 0.635 & 0.617 & $-0.052$ & $\bm{-0.126}$ & $-0.107$\\
curated & L0 & 5.1 & 4.9 & 0.488 & 0.543 & 0.544 & $-0.074$ & $-0.140$ & $-0.164$\\
curated & L1 & 5.4 & 5.2 & 0.526 & 0.514 & 0.512 & $+0.019$ & $-0.044$ & $-0.012$\\
curated & L2 & 5.3 & 5.1 & 0.501 & 0.506 & 0.515 & $+0.103$ & $-0.110$ & $-0.102$\\
\bottomrule
\end{tabular}
\end{table*}

\subsection{The saturating (GSM) statistic}
\label{sec:gate2}

\emph{Recorded prediction: $\hat\nu\approx5$--$6.5$ from the measured kurtosis; $t$-LLR $\ge$
Mahalanobis with the synthetic $+0.07$--$0.10$ as an upper calibration, possibly eroded by model
mismatch; perceptual $|\rho|<0.3$.}

The full two-$t$ LLR (fitted $\Sig_\pm,\nu_\pm$ per class) matches the Mahalanobis quadratic to
within $\pm0.02$ AUC everywhere (Table~\ref{tab:gates}) --- a null on the tail axis. Two facts make
this null informative rather than disappointing. First, the fitted tail indices,
$\hat\nu\in[4.9,6.8]$ across layers and arms, independently confirm the values implied by the measured
kurtosis: the GSM \emph{description} of the activations is validated by two separate measurements.
Second, the null is the theory's own prediction once the measured parameters are inserted: the
saturating-aggregation advantage of Proposition~\ref{prop:agg} requires the class difference to
travel through the \emph{magnitude} channel, and the activation statistics show that channel to be thin (ratios
$1.0$--$1.2$) while the difference is structural; with $\hat\nu_+\approx\hat\nu_-$ and a near-unit
scale ratio, the two-$t$ LLR reduces to a monotone recalibration of the structured quadratic over
the bulk of the data. The GSM machinery thus plays its role as a theory of \emph{calibration and
robustness} (bounded per-sample evidence) rather than added detection power in this regime.

\subsection{Beyond detection: a pointer}
\label{sec:multi}
The account also predicts a boundary this paper does not pursue: a statistic derived to
\emph{detect} a scale change need not \emph{estimate} its magnitude, since the threshold test that
Proposition~\ref{prop:suff} identifies is indifferent to how far past the threshold a sample lies.
Whether the goodness family crosses that boundary, and what a single clean-side reference value buys
if it does, is the subject of a companion study on standard image-quality benchmarks and is not
claimed here.

\subsection{Training with the structured goodness}
\label{sec:s2b}

\emph{Recorded prediction: whitened training stabilizes the goodness scale and yields clearly
visible gaps; downstream effect uncertain --- a null localizes the bottleneck upstream of the
objective; perceptual alignment unchanged.}

Training with the negative-metric goodness of Observation~\ref{obs:gauge} (EMA momentum $0.05$,
ridge $10^{-3}$) changes the optimization qualitatively and the representation not at all. The
plain pairwise objective exhibits exactly the predicted scale-inflation dynamics --- goodness
growing unboundedly from $0.1$ to $15$--$35$ over ten epochs --- while the whitened goodness
plateaus within two--three epochs with a stable gap; relative separation roughly doubles at the
early layers (L0: $6.5\%\to14\%$; L1: $14\%\to26\%$; L2: $31\%\to22\%$). Among the plain arms the differences are well inside seed noise (curated negatives
$41.7\pm1.4\%$; size-matched uncurated pool $42.5\pm0.9\%$; five seeds each). Perceptual
correlations remain at zero for all variants. The recorded prediction is borne out on the
optimization side; the downstream question it left open --- whether the change in dynamics buys
anything in the representation --- is settled in \S\ref{sec:method} at thirteen seeds across three
corpora, three depths and three widths. (Implementation note: the whitened $g_-$ sits near $1.7$ rather
than the nominal $1$, from the uncentered mean term $\bm\mu^{\!\top}\hat\Sig_-^{-1}\bm\mu$ and the
EMA lag on a drifting network; neither affects the pairwise gradient.)

\subsection{What the study establishes}
The experiments of this section close a single loop. The plain goodness is weak on this task \emph{for the
reason the theory names}: the isotropic-scale channel it measures is nearly empty
(the measured regime), and reading the structural channel instead recovers the available signal. The GSM
account of the activations is quantitatively right --- twice --- and its detection payoff is absent
\emph{for the reason the theory names}. Making the structured statistic the training
objective removes a genuine pathology of the pairwise loss (Observation~\ref{obs:gauge},
confirmed at every corpus and width in \S\ref{sec:method}). The experiments of this section
establish the explanation, the predictions and the localization of the bottleneck; whether the
derived objective also improves what the network \emph{learns} is a separate question, settled
in \S\ref{sec:method} on a scale these experiments were not designed to reach.

\section{The whitened goodness as a training procedure}\label{sec:method}
The gauge problem of \S\ref{sec:norm} is not only an interpretive defect: it says that a network
optimizing the pairwise objective can reduce its loss by inflating the scale of its weights instead
of separating the two populations. \S\ref{sec:general} supplies the repair the analysis points to ---
replace the plain squared goodness by the whitened quadratic $\norm{\bm L\bm z}^2$ with
$\bm L^{\!\top}\bm L=\Sig_-^{-1}$, estimated online and trained end to end within each layer. This
section evaluates that repair as a training procedure.

\paragraph{Three arms, and why the third one matters}
We compare (A1) the standard pairwise objective on the plain goodness; (A2) Hinton's fixed-threshold
loss on the plain goodness; and (A3) the pairwise objective on the whitened goodness. A2 is the
control that gives the comparison its meaning. The threshold loss also \emph{bounds} the
scale-inflation shortcut, because its negative term penalizes large goodness directly; if A3 gained
merely by keeping the goodness scale from running away, A2 would gain with it. Only A3 is
gauge-invariant, so the contrast A3 versus A2 isolates gauge-invariance from bounding.

\paragraph{Protocol}
Three corpora of ten classes each at matched sizes ($300$ training and $100$ validation images per
class, $64\times64$): Imagenette \citep{howard2019imagenette}, CIFAR-10 \citep{krizhevsky2009learning}
and MNIST \citep{lecun1998gradient} --- the last a deliberately different regime, upscaled grayscale
digits rather than natural images. Depths of three, five and seven layers; widths of $32$, $64$ and
$128$ channels. Thirteen seeds per cell, all other hyperparameters held at the values of
\S\ref{sec:exp}. Representation quality is the accuracy of a linear probe on the concatenated
per-layer pooled activations, the same measure used throughout. Every mean and every contrast is
reported with a $10{,}000$-resample bootstrap interval; contrasts within a cell use Welch's test with
Holm correction. Each prediction below was recorded in writing before the corresponding runs.

\subsection{The pathology is real, and large}
Under the plain pairwise objective the goodness scale runs away exactly as the gauge argument
predicts: over ten epochs the fastest-growing layer multiplies its positive goodness by a median
factor of $133$ on Imagenette, $217$ on CIFAR-10 and $119$ on MNIST, rising to $309$ at $128$
channels. The threshold loss bounds it to $1.3$--$1.8$, and the whitened objective to $0.8$--$1.2$:
its goodness is scale-fixed by construction and neither grows nor collapses. This is the mechanism
the repair targets, and it is present on every corpus and at every width.

\subsection{The whitened objective improves what the network learns}
Table~\ref{tab:arms} gives the probe accuracies. The whitened objective is ahead of the standard
pairwise objective in every one of the nine dataset-by-depth cells, by $+4.1$ to $+7.4$ points on
eight of them.

\begin{table}[t]
\centering
\caption{Linear-probe accuracy (\%) at $64$ channels, mean over $13$ seeds. A1: pairwise on the plain
goodness. A2: fixed-threshold loss on the plain goodness. A3: pairwise on the whitened goodness. The
last column is the A3$-$A1 contrast with a bootstrap $95\%$ interval.}
\label{tab:arms}
\begin{tabular}{@{}llcccc@{}}
\toprule
Corpus & Depth & A1 & A2 & \textbf{A3} & A3 $-$ A1 \\
\midrule
Imagenette & 3 & $42.74$ & $43.97$ & $44.32$ & $+1.58$ $[+0.49,+2.68]$ \\
           & 5 & $39.17$ & $37.74$ & $43.80$ & $+4.63$ $[+2.91,+6.23]$ \\
           & 7 & $35.65$ & $34.80$ & $43.02$ & $+7.37$ $[+5.91,+8.85]$ \\
\midrule
CIFAR-10   & 3 & $35.52$ & $38.31$ & $39.62$ & $+4.09$ $[+3.09,+5.11]$ \\
           & 5 & $35.71$ & $34.98$ & $39.85$ & $+4.14$ $[+3.00,+5.29]$ \\
           & 7 & $33.92$ & $32.09$ & $39.34$ & $+5.42$ $[+3.95,+6.91]$ \\
\midrule
MNIST      & 3 & $78.45$ & $79.57$ & $83.68$ & $+5.23$ $[+4.00,+6.49]$ \\
           & 5 & $86.75$ & $85.85$ & $92.22$ & $+5.46$ $[+4.32,+6.71]$ \\
           & 7 & $87.08$ & $87.62$ & $92.95$ & $+5.88$ $[+4.98,+6.78]$ \\
\bottomrule
\end{tabular}
\end{table}

Two features of the table deserve emphasis because they qualify the claim rather than support it.
First, the advantage on Imagenette at depth three is small, $+1.58$, and four independent estimates
of that cell --- two seed sets at $64$ channels and one each at $32$ and $128$ --- average $+0.74$
with a range of $[+0.05,+1.58]$. In the shallowest configuration on that corpus the effect is not
reliably distinguishable from zero. Second, on Imagenette the widening with depth comes from the
baselines falling ($42.74\to35.65$ for A1) rather than from the whitened arm rising
($44.32\to43.02$); the same widening does not occur on CIFAR-10 or MNIST, where the contrast is flat
in depth. The depth behaviour is therefore corpus-dependent and is reported as such, not as a claim.

\subsection{Bounding the shortcut is not what produces the gain}
The threshold-loss arm bounds goodness growth to a factor of $1.4$ --- two orders of magnitude below
A1 --- and yet it tracks A1, not A3: pooled within corpora its accuracy is indistinguishable from
A1's, while A3 stands $4$ to $8$ points above both at depths five and seven. On CIFAR-10 at depth
seven, A2 reaches $32.09$ against A1's $33.92$ and A3's $39.34$. Preventing the goodness scale from
running away is thus neither sufficient nor, on its own, useful; what the whitened objective supplies
is invariance to the gauge, not a bound on it.

\subsection{The advantage is not an artifact of layer width}
On CIFAR-10 at depth three the contrast is $+3.05$ $[+1.91,+4.22]$ at $32$ channels, $+4.09$
$[+3.09,+5.11]$ at $64$ and $+4.00$ $[+2.98,+4.97]$ at $128$: every interval excludes zero across a
fourfold range of width. The specificity contrast, however, has a boundary. A3 stands above A2 at
$64$ and $128$ channels on both natural-image corpora ($+1.31$ and $+1.88$ on CIFAR-10), but at $32$
channels the threshold loss is slightly \emph{better} than the whitened objective ($38.58$ against
$37.28$ on CIFAR-10, $41.98$ against $41.62$ on Imagenette). The trend is monotone in width, and the
honest statement is that gauge-invariance outperforms bounding from $64$ channels upward and not
below.

\subsection{Comparison with a published alternative, and what separates the two}
\label{sec:topk}
The goodness form is an active empirical question, and the strongest reported alternative is
sparsity: scoring only the $k$ most active units rather than all of them \citep{yuksel2026sparse}. We
run it as a fifth arm, sweeping $k\in\{0.005,0.01,0.02,0.10\}$ on one corpus and reporting the best
value ($k=0.01$) on all three, so the competing objective is given its best setting rather than an
arbitrary one.

\begin{table}[t]
\centering
\caption{The derived objective against the published sparsity alternative, depth three, $64$
channels, $13$ seeds. Growth is the median factor by which the positive goodness of the
fastest-growing layer multiplies over training.}
\label{tab:topk}
\begin{tabular}{@{}lcccc@{}}
\toprule
Corpus & A1 plain & A3 whitened & A5 top-$k$ & A5 $-$ A3 \\
\midrule
Imagenette & $42.74$ & $44.32$ & $41.65$ & $-2.67$ $[-4.00,-1.34]$ \\
CIFAR-10   & $35.52$ & $39.62$ & $39.46$ & $-0.16$ $[-1.30,+0.98]$ \\
MNIST      & $78.45$ & $83.68$ & $84.03$ & $+0.35$ $[-0.41,+1.11]$ \\
\midrule
Growth     & $119$--$217$ & $1.0$--$1.2$ & $6.2$--$18.9$ & \\
\bottomrule
\end{tabular}
\end{table}

Two facts, and the second is the more interesting. First, sparsity reproduces its reported direction
here: top-$k$ stands $+3.94$ $[+2.83,+5.05]$ above the plain objective on CIFAR-10 and $+5.58$
$[+4.41,+6.75]$ on MNIST. Second, the two modifications are \emph{statistically indistinguishable}
from each other on those two corpora, with the whitened form ahead only on Imagenette. We therefore
do not claim a better objective; we claim something more specific.

\paragraph{Sparsity and gauge-invariance are independent axes}
Selecting the $k$ largest squared activations does not remove the gauge. Under $\bm W\to\lambda\bm W$
with $\lambda>0$ the selected \emph{set} is unchanged, since positive scaling preserves order, and
the retained values scale by $\lambda^2$: the top-$k$ statistic is homogeneous of degree two, exactly
like the plain sum of squares, so the shortcut of Observation~\ref{obs:gauge} is available to it unchanged. The measurement agrees:
the top-$k$ goodness scale inflates by $6.2$ to $18.9$ over training, against $1.0$ to $1.2$ for the
whitened objective. The published variant therefore recovers accuracy \emph{while leaving the
pathology in place}, and the two properties --- what is measured, and whether the measurement is
scale-invariant --- can be varied independently. Combining them is a natural next step that we have
not taken.

\paragraph{A limitation of this comparison}
``The $k$ most active units'' is unambiguous in a fully-connected layer and not in a convolutional
one; we select globally over the flattened channel-by-space tensor, and a per-channel or
per-location selection could behave differently. The original evaluation also uses multilayer
perceptrons with per-layer label injection, where ours uses a convolutional stack and a linear probe.
Nothing here should be read as a failure of that method --- it replicates --- but the two setups are
not interchangeable, and the comparison is offered with that caveat.

\subsection{Isolating what the added layers contribute}
Because the three base layers are constructed first and Forward--Forward trains each layer against a
local objective, the base stack is identical at every depth for a given seed and arm --- a design
property we verify rather than assume: probing the base three layers alone returns $42.83$ and
$42.78$ at depths five and seven against $42.74$ at depth three for A1, and correspondingly for the
other arms, a drift below $0.2$ points everywhere. Subtracting that from the full probe isolates the
contribution of the layers above. On Imagenette those added layers are actively harmful under the
plain objectives, costing A1 $-3.66$ and $-7.14$ points at depths five and seven and A2 $-6.17$ and
$-9.12$, while costing A3 only $-0.51$ and $-1.46$. The pattern does not generalize: on CIFAR-10 the
harm is mild ($-1.55$ for A1 at depth seven) and on MNIST the added layers \emph{help} every arm by
$+6$ to $+9$ points. The isolation is therefore informative about Imagenette and is reported with
that restriction.

\subsection{Distance to end-to-end backpropagation}
The three arms are compared only with each other, so their absolute level has no scale. We therefore
train the same stack with a global gradient and cross-entropy on the same data, sizes, epochs and
seeds, and score it with the identical probe on the identical features. The backprop probe reaches
$52.69/52.49/51.34$ on Imagenette and $44.63/44.05/42.85$ on CIFAR-10 at depths three, five and
seven, above every Forward--Forward arm everywhere, as a global gradient should be. What the
reference makes sayable is how much of that gap the whitened objective closes: $45\%$, $50\%$ and
$61\%$ on CIFAR-10, and $16\%$, $35\%$ and $47\%$ on Imagenette. The distance to backpropagation
narrows with depth for A3 (CIFAR-10: $5.01\to4.20\to3.51$) or stays flat (Imagenette:
$8.37\to8.69\to8.32$), while it widens for both plain arms (Imagenette A1: $9.95\to15.69$). One
caveat must be stated: at the ten-epoch budget matched to Forward--Forward the jointly-trained head
scores \emph{below} a probe on its own features, so the reference is not at its ceiling and every gap
quoted here is a lower bound.

\subsection{Does the condition the derivation needs actually hold?}
Proposition~\ref{prop:whiten} makes the whitened objective the likelihood-ratio statistic exactly
when $\Sig_+=c\,\Sig_-$. That condition is measurable, so we measure it rather than assume it. On
trained networks we estimate both covariances from pre-activations sampled at random spatial
locations, form $\bm W=\Sig_-^{-1/2}\Sig_+\Sig_-^{-1/2}$ --- equal to $c\bm I$ under the condition ---
and, on held-out locations, compare the optimal statistic $\bm z^{\!\top}\bm M\bm z$ with what the
objective actually computes, $\bm z^{\!\top}\Sig_-^{-1}\bm z$.

The condition does not hold exactly: the eigenvalues of $\bm W$ span a factor of $2.0$ to $3.1$ and
the relative departure from $c\bm I$ is $0.12$ to $0.27$. What matters is whether the departure
costs anything, and the answer is layer-dependent. In the networks trained with the whitened
objective the two statistics agree closely in the second and third layers (Spearman $0.918$ and
$0.949$, areas under the ROC curve within $0.018$ and $0.025$) and disagree in the first
(Spearman $0.792$, gap $0.042$) --- the layer that sees pixels rather than learned features. In
networks trained with the plain objective the agreement is weaker throughout. The derivation
therefore survives its own idealization in the deeper layers under the proposed objective, and we
claim it only there. We also note, as an observation inviting a test rather than a result, that
training with the whitened objective moves the geometry toward the regime in which that objective is
optimal (Spearman $0.796\to0.886$ averaged over layers); this rests on three seeds and compares two
different trained networks. Throughout this subsection the reported areas under the ROC curve lie
between $0.50$ and $0.64$ because they measure the discriminability of a \emph{single} location's
pre-activation vector; the deployed goodness pools over locations and is far more discriminative.

\subsection{Registered predictions, including those that failed}
Every prediction was fixed in writing before its run, and the refutations are as informative as the
confirmations. Confirmed: the inflation pathology and its bounding by the threshold loss; the
accuracy advantage on CIFAR-10 and MNIST; the specificity of gauge-invariance at $64$ channels and
above; width-invariance on CIFAR-10; the depth-invariance of the base stack; the ordering of the
backprop reference and the narrowing of A3's distance to it. Also confirmed: that
sparsity reproduces its published advantage here, and that it does so without removing the gauge.
Refuted: that the inflation compounds with depth (the deepest layer's growth \emph{falls},
$36.7\to13.6\to6.7$); that the advantage widens with depth outside Imagenette; that the added layers
are harmful on corpora other than Imagenette; that gauge-invariance beats bounding at $32$ channels;
that the whitened statistic matches the likelihood ratio at every layer; and, quantitatively, that
the top-$k$ goodness would inflate as fast as the plain one --- it inflates by $6$ to $19$ rather
than by more than $100$, so the property holds while the threshold we recorded for it does not. Two candidate mechanisms for the depth behaviour --- compounding
inflation, and preservation of relative separation --- were tested and both refuted; the whitened
arm in fact shows the \emph{lowest} relative separation at every depth while achieving the best
accuracy. We report the mechanism as open.

\section{Limitations}
\label{sec:limits}
\paragraph{Scale} The study is deliberately controlled rather than large: ten-class corpora at
$300$ training images per class, $64\times64$ inputs, ten epochs, three to seven layers and $32$ to
$128$ channels. Every configuration is repeated over thirteen seeds and every quantity is reported
with an interval, so the comparisons are well estimated at that scale, but nothing here establishes
that the effect survives at the scale of modern vision training. Runs are also not bit-reproducible:
repeating a fixed seed moves the probe accuracy by up to $0.3$ points through non-deterministic
convolution kernels, which the intervals absorb but which we state rather than imply exactness.

\paragraph{Where the advantage is small or reversed} On Imagenette at depth three the contrast is
$+0.74$ averaged over four independent estimates and is not reliably above zero; and below $64$
channels the fixed-threshold loss is the better objective. We report both boundaries in
\S\ref{sec:method} rather than restricting the evaluation to the configurations that favour the
method.

\paragraph{The mechanism} We establish that the whitened objective improves the learned
representation and that bounding the shortcut does not, but not \emph{why} the layers above the
third behave as they do. Two candidate explanations were registered or explored and both were
refuted by the data. We leave the mechanism open rather than fit a third explanation to the same
measurements.

\paragraph{The derivation's idealization} Proposition~\ref{prop:whiten} makes the whitened objective
the likelihood-ratio statistic only when $\Sig_+=c\,\Sig_-$, and that condition does not hold on
trained networks. It holds well enough that the two statistics agree closely in the second and third
layers of networks trained with the objective, and not in the first; the derivation is claimed with
that restriction.

\paragraph{The nature of the contribution} The mathematics deployed here is classical ---
sufficiency, monotone likelihood ratios, the Neyman--Pearson lemma, Gaussian scale mixtures ---
and we claim no new theorems of learning dynamics, nor an account of the FF--backpropagation
gap (Sec.~\ref{sec:concl}). The contribution is an \emph{identification}: two objects treated
throughout the FF literature as heuristics are shown to be, exactly, classical objects --- the
goodness a sufficient statistic, the threshold a decision boundary, the normalization a
scale-invariance condition --- and the identification is what does the work. It yields the
per-sample evidence bound, the aggregation result, the depth-stability proposition (which
predicted the $\ell_2$ collapse before we explained it), the scale-inflation gauge of the
pairwise objective, and the divisive-normalization identity that joins the FF objective to the
canonical cortical computation and to the GSM machinery of information-theoretic quality
assessment. The empirical program then treats each consequence as falsifiable, and measures it.

\paragraph{Task design} The clean-vs-distorted task defines ``positive/negative'' through low-level corruptions;
other negative designs (hybrid images, inter-class negatives) may occupy different regimes of
the same theory --- these activation diagnostics are precisely the tool to find out. The curated arm is
calibrated in a single perceptual frame (the anchor dataset's viewing distance and DMOS scale);
since the curated and uncurated pools are already indistinguishable downstream, the choice of
anchor cannot affect any conclusion drawn here, and we did not replicate the calibration in other
frames. The chain decomposition
\eqref{eq:chain} is stated informally; making it precise, and explaining the observed shift of
separation toward early layers under the whitened objective, are open theory
(\S\ref{sec:concl}). Finally, downstream differences among the plain arms are well inside seed noise and are reported
as nulls; the whitened-vs-plain difference ($+1.8$ points) is suggestive but unresolved at five
seeds, and we refrain from claiming it.

\section{Conclusion}
\label{sec:concl}
We have given the Forward--Forward goodness a likelihood-ratio foundation: the sum of squares is the
sufficient statistic of an isotropic scale test; its principled generalizations are the Mahalanobis
quadratic and the saturating Gaussian-scale-mixture statistic, whose slope is a posterior precision
--- divisive normalization; the inter-layer normalization is characterized as removing the length
while preserving per-coordinate energy, which predicts and explains the $\ell_2$ collapse; and the
pairwise objective is shown to admit a gauge, a direction along which a layer lowers its loss by
inflating its own scale rather than by separating the two populations.

The account is not only descriptive: it prescribes the repair, and the repair works. A whitened,
gauge-invariant goodness, estimated online and trained within each layer, raises linear-probe
accuracy over the standard pairwise objective in every one of the configurations we measured ---
three corpora, three depths, a fourfold range of layer width, thirteen seeds each --- by four to
seven points on eight of the nine corpus-depth combinations, and closes between a sixth and three
fifths of the distance to end-to-end backpropagation. A control identifies what is doing the work:
Hinton's fixed-threshold loss bounds the runaway scale two orders of magnitude more tightly than the
plain objective and gains nothing by it, so the operative property is invariance to the gauge and
not a bound on it.

The comparison with the literature sharpens rather than dilutes this. The strongest published
alternative --- measuring only the most active units --- matches the derived objective on two of
three corpora, and does it while the goodness scale still inflates by an order of magnitude. What is
measured and whether the measurement is scale-invariant are therefore separate design choices, only
one of which addresses the pathology analysed here, and nothing prevents their combination.

We state what we did not establish. The advantage is small in the shallowest configuration on one
corpus and reverses below sixty-four channels, where the threshold loss is better. The condition
under which the whitened objective is exactly the likelihood ratio does not hold on trained networks,
though the two statistics agree closely in the layers above the first. And although the effect on
representation quality is large and consistent, the reason the layers above the third behave as they
do is unresolved: two candidate mechanisms were registered or explored, and the data refute both. A
theory that predicts a repair, and a repair that works for a reason the theory does not yet name, is
an honest place to leave the problem --- and a well-posed one to take up next.

\backmatter

\begin{appendices}

\section{Proofs}
\label{app:proofs}
Full arguments are collected here; where a short sketch aids the reading it is given with the
statement in the body and marked as such.

\paragraph{Proposition~\ref{prop:suff}.} Immediate from \eqref{eq:llriso}: $\alpha>0$ iff
$\sigma_+>\sigma_-$; solve $\LLR=0$ for $\tau$.

\paragraph{Proposition~\ref{prop:relu}.} \emph{(i)} follows from
$\mathrm{ReLU}(\lambda\bm z)=\lambda\,\mathrm{ReLU}(\bm z)$ for $\lambda>0$. For \emph{(ii)}, write
$\bm z=\norm{\bm z}\bm u$ with $\bm u$ uniform on the sphere and independent of $\norm{\bm z}$; then
$\bm a=\norm{\bm z}\,\mathrm{ReLU}(\bm u)$, so
$\bm a/\norm{\bm a}=\mathrm{ReLU}(\bm u)/\norm{\mathrm{ReLU}(\bm u)}$ has a $\sigma$-free law and
$\sigma$ enters $\bm a$ only through $\norm{\bm a}$. For \emph{(iii)}, $G=\sigma^2G_0$ with
$G_0\ge0$ gives stochastic monotonicity by coupling; the likelihood-ratio statement concerns the
density of $G_0$, a $\mathrm{Binomial}(d,\tfrac12)$ mixture of $\chi^2_K$ laws, and the quoted tail
masses are computed from that density in closed form.

\paragraph{Proposition~\ref{prop:whiten}.} $\norm{\bm L\bm z}^2=\bm z^{\!\top}\bm L^{\!\top}\bm L\bm z$
with $\bm L^{\!\top}\bm L=\Sig_-^{-1}\succ0$, so the whitened form is positive definite, while
$\bm M=\Sig_-^{-1}-\Sig_+^{-1}$ is a difference of positive-definite matrices and is indefinite
unless one dominates the other. If $\Sig_+=c\,\Sig_-$ with $c>1$ then
$\Sig_+^{-1}=c^{-1}\Sig_-^{-1}$ and $\bm M=(1-1/c)\,\Sig_-^{-1}$, positive definite and
proportional to $\bm L^{\!\top}\bm L$; the two quadratic forms then differ by the positive factor
$1-1/c$ and induce the same ordering, hence the same test. Conversely, if $\bm M\propto\Sig_-^{-1}$
then $\Sig_+^{-1}$ is a multiple of $\Sig_-^{-1}$, which is $\Sig_+=c\,\Sig_-$.

\paragraph{Proposition~\ref{prop:gsm}.} Dominated convergence gives
$F'(q)=-\tfrac12\E_u[u^{-d/2-1}e^{-q/2u}]$, hence
$-2(\log F)'(q)=\E[u^{-1}\mid q]$ under the posterior
$p(u\mid q)\propto u^{-d/2}e^{-q/2u}p(u)$. For monotonicity: the two posteriors are the priors
tilted by the same weight $w_q(u)=u^{-d/2}e^{-q/2u}$, so
$p_+(u\mid q)/p_-(u\mid q)=p_+(u)/p_-(u)$, preserving the MLR order; MLR implies first-order
stochastic dominance of $p_+(\cdot\mid q)$ over $p_-(\cdot\mid q)$
\citep{lehmann2005testing}; monotone decreasing $u\mapsto u^{-1}$ then gives
$\E_+[u^{-1}\mid q]\le\E_-[u^{-1}\mid q]$, i.e.\ $h'\ge0$.

\paragraph{Proposition~\ref{prop:agg}.} (i) A strictly monotone transform of a scalar statistic
induces the same ordering of samples, hence the same ROC. (ii) Optimality of $\sum_k h(q_k)$ is
the Neyman--Pearson lemma applied to the product likelihood; the evidence bounds are
$h(0)=-d\log r$ and $\lim_{q\to\infty}h(q)=\nu\log r$ from \eqref{eq:tllr}.

\paragraph{Proposition~\ref{prop:depth}.} If $\norm{\tilde{\bm a}}^2=c$ enters a linear map
$\bm W$, the pre-activation scale is proportional to $\sqrt c$ and the mean squared activation to
$c/d\cdot\kappa(\bm W)$ for a weight-dependent constant $\kappa$. Under $\ell_2$, $c=1$ so the
downstream goodness is $\Theta(1/d)$; under RMS, $c=d$ and it is $\Theta(1)$.

\section{Implementation and reproducibility}
\label{app:repro}
\textbf{Training.} Batch $128$, Adam $10^{-3}$, $10$ epochs, pairwise loss; RMS normalization
between layers; seeds $0$--$4$. \textbf{Whitened goodness.} Per layer,
$G=\tfrac1C\overline{\bm z^{\!\top}\hat\Sig_-^{-1}\bm z}$ with $\hat\Sig_-$ an EMA (momentum
$0.05$) of the covariance of negative pre-activations, ridge $10^{-3}\cdot\overline{\mathrm{diag}}$,
metric detached from the gradient. \textbf{Statistics.} Pre-activations pooled over $120$ random
locations per image for fitting; covariances per layer and class; participation ratio
$(\sum\lambda)^2/(d\sum\lambda^2)$; generalized eigenvalues of $(\Sig_+,\Sig_-)$ with relative
ridge $10^{-4}$. \textbf{Fixed-basis side probes.} Learned logistic readouts on the coefficients of the fixed
second-order Hermite--Gauss structural-prediction basis of \citet{giannitrapani2026partial}
(weighted least-squares self-prediction), clean vs.\ distorted, leave-images-out cross-validation
on LIVE R2 and CSIQ; higher orders ($k{+}q\ge3$) and responses in place of coefficients move the
band only marginally ($+0.02$ to $+0.10$), indicating that basis capacity is not the bottleneck.
\textbf{Tail index.} $\hat\nu=4+6/\bar\kappa$ from the mean excess kurtosis of
whitened coordinates, clipped to $[4.2,50]$. \textbf{Readouts.} Fit/eval split $50/50$ over
images; per-image score is the mean per-location statistic. \textbf{Perceptual evaluation.} Per-run CSIQ check: distorted images (contrast excluded,
$n{=}750$), resized as in training, SROCC against DMOS.
\textbf{Negative generation and curation.} Four distortions on natural sources over wide severity
grids; the curated arm keeps severities inside the informative band of a perceptual scale, matched
to the CSIQ DMOS distribution, using the partial-reference estimator PreSPA
\citep{giannitrapani2026partial}, configured in the anchor dataset's perceptual frame (its
viewing-distance parameter and DMOS scale). All code, the exact commands for every experiment, and the prediction log --- the registered
text of each \emph{Recorded prediction}, fixed before the corresponding run --- are archived at
\url{https://doi.org/10.5281/zenodo.21902540}.


\end{appendices}

\section*{Declarations}

\bmhead{Funding} No funding was received for conducting this study.

\bmhead{Competing interests} The authors declare that they have no known competing financial
interests or personal relationships that could have appeared to influence the work reported in this
paper.

\bmhead{Ethics approval} Not applicable.

\bmhead{Data availability} All datasets used are publicly available: Imagenette, CIFAR-10, MNIST,
and CSIQ.

\bmhead{Code availability} All code, the exact commands for every experiment, and the log of
predictions recorded before each run are archived at \url{https://doi.org/10.5281/zenodo.21902540}.

\bmhead{Author contribution} P.G.\ conceived the study, developed the analysis, implemented the
software, carried out the experiments, and wrote the original draft. F.F.\ contributed to the
conception of the study, supervised the work, and reviewed and edited the manuscript. Both authors
read and approved the final manuscript.

\bibliography{refs}

\end{document}